\documentclass{article}
\usepackage{ijcai23}

\usepackage{times}
\usepackage{soul}
\usepackage{url}
\usepackage[hidelinks]{hyperref}
\usepackage[utf8]{inputenc}
\usepackage[small]{caption}
\usepackage{graphicx}
\usepackage{amsmath}
\usepackage{amsthm}
\usepackage{booktabs}
\usepackage[ruled,vlined,linesnumbered]{algorithm2e}

\usepackage[switch]{lineno}

\usepackage{xspace}
\usepackage{color}

\usepackage{amsfonts}
\usepackage{amssymb}
\usepackage{amstext}
\usepackage{mathtools}

\usepackage{dsfont}

\usepackage{todonotes}

\newtheorem*{theorem*}{Theorem}
\newtheorem{theorem}{Theorem}
\newtheorem{lemma}[theorem]{Lemma}

\newcommand{\oea}{\mbox{${(1 + 1)}$~EA}\xspace}

\newcommand{\NSGA}{NSGA\nobreakdash-II\xspace}

\newcommand{\onemax}{\textsc{OneMax}\xspace}
\newcommand{\LO}{\textsc{Leading\-Ones}\xspace}
\newcommand{\leadingones}{\LO}

\newcommand{\oneminmax}{\textsc{OneMinMax}\xspace}
\newcommand{\cocz}{\textsc{COCZ}\xspace}
\newcommand{\lotz}{\textsc{LOTZ}\xspace}
\newcommand{\ojzj}{\textsc{OneJumpZeroJump}\xspace}

\DeclareMathOperator{\ex}{ex}
\DeclareMathOperator{\total}{total}

\DeclareMathOperator{\init}{init}
\DeclareMathOperator{\elim}{elim}
\DeclareMathOperator{\Temp}{Temp}

\newcommand{\R}{\ensuremath{\mathbb{R}}}

\newcommand{\N}{\ensuremath{\mathbb{N}}} 

\newcommand{\bbone}[1]{\mathds{1}_{#1}}

\newcommand{\calP}{\ensuremath{\mathcal{P}}}

\newcommand{\calO}{\ensuremath{\mathcal{O}}}
\newcommand{\calV}{\ensuremath{\mathcal{V}}} 
\newcommand{\calW}{\ensuremath{\mathcal{W}}}

\let\originalleft\left
\let\originalright\right
\renewcommand{\left}{\mathopen{}\mathclose\bgroup\originalleft}
\renewcommand{\right}{\aftergroup\egroup\originalright}

\title{Runtime Analyses of Multi-Objective Evolutionary Algorithms\\ in the Presence of Noise\thanks{Author-generated version.}
}

\author{
Matthieu Dinot$^1$
\and
Benjamin Doerr$^{2}$\and
Ulysse Hennebelle$^{1}$\and
Sebastian Will$^{2}$
\affiliations
$^1$\'Ecole Polytechnique, Institut Polytechnique de Paris, Palaiseau, France\\
$^2$Laboratoire d'Informatique (LIX), CNRS, \'Ecole Polytechnique, Institut Polytechnique de Paris, Palaiseau, France
\emails
matthieu.dinot@polytechnique.edu,
benjamin.doerr@polytechnique.edu,
ulisse.hennebelle@polytechnique.edu,
sebastian.will@polytechnique.edu
}

\begin{document}

\maketitle

\begin{abstract}
In single-objective optimization, it is well known that evolutionary algorithms also without further adjustments can stand a certain amount of noise in the evaluation of the objective function. In contrast, this question is not at all understood for multi-objective optimization.

 In this work, we conduct the first mathematical runtime analysis of a simple multi-objective evolutionary algorithm (MOEA) on a classic benchmark in the presence of noise in the objective function. 
 We prove that when bit-wise prior noise with rate $p \le \alpha/n$, $\alpha$ a suitable constant, is present, the \emph{{simple} evolutionary multi-objective optimizer} (SEMO) without any adjustments to cope with noise finds the Pareto front of the OneMinMax benchmark in time $O(n^2\log n)$, just as in the case without noise. Given that the problem here is to arrive at a population consisting of $n+1$ individuals witnessing the Pareto front, this is a surprisingly strong robustness to noise (comparably simple evolutionary algorithms cannot optimize the single-objective OneMax problem in polynomial time when $p = \omega(\log(n)/n)$). Our proofs suggest that the strong robustness of the MOEA stems from its implicit diversity mechanism designed to enable it to compute a population covering the whole Pareto front. 
 
 Interestingly this result only holds when the objective value of a solution is determined only once and the algorithm from that point on works with this, possibly noisy, objective value. We prove that when all solutions are reevaluated in each iteration, then any noise rate $p = \omega(\log(n)/n^2)$ leads to a super-polynomial runtime. This is very different from single-objective optimization, where it is generally preferred to reevaluate solutions whenever their fitness is important and where examples are known such that not reevaluating solutions can lead to catastrophic performance losses.
\end{abstract}

\section{Introduction}
Many real-world optimization problems consist of multiple, often conflicting objectives. For these, a single optimal solution usually does not exist. Consequently, a common solution concept is to compute a set of solutions which cannot be improved in one objective without worsening in another one (Pareto optima) and then let a decision maker select one of these. 

Due to their population-based nature, evolutionary algorithms (EAs) are an obvious heuristic approach to such problems, and in fact, such multi-objective evolutionary algorithms (MOEAs) have been successfully applied to many multi-objective problems~\cite{CoelloLV07,ZhouQLZSZ11}. 

Evolutionary algorithms are known to be robust against different types of stochastic disturbances such as noise or dynamic changes of the problem instance~\cite{JinB05}. Surprisingly, as regularly pointed out in the literature~\cite{LiefoogheBJT07,Gutjahr12,GutjahrP16}, only very little is known on how MOEAs cope with such stochastic optimization problems. In particular, while it is known that single-objective evolutionary algorithms without any specific adjustments can stand a certain amount of noise in the evaluation of the objective function, we are not aware of any such result in multi-objective optimization.

We approach this question via the methodology of mathematical runtime analysis~\cite{NeumannW10,AugerD11,Jansen13,ZhouYQ19,DoerrN20}. This field, for more than twenty years, has significantly enlarged our understanding of the working principles of all kinds of randomized search heuristics, including both MOEAs~\cite{Brockhoff11bookchapter} and single-objective evolutionary algorithms solving stochastic optimization problems~\cite{NeumannPR20bookchapter}. Despite this large amount of work, there is not a single runtime analysis discussing how a standard MOEA computes or approximates the Pareto front of a multi-objective problem in the presence of noise (and this is what we shall do in the present work). The only paper that conducts a mathematical runtime analysis of a MOEA in a noisy setting analyzes a combination of the adaptive Pareto sampling frameworks with the simple evolutionary multi-objective optimizer (SEMO)~\cite{Gutjahr12}, so this algorithm definitely is not anymore a standard MOEA.

To start closing this gap, we conduct a runtime analysis of a simple MOEA, namely the SEMO, on the classic benchmark problem \oneminmax, in the presence of one-bit prior noise. We prove that when the noise rate is at most $\alpha/n$, where $\alpha$ is a suitable constant, then the population of this MOEA, despite the noise, after an expected number of $O(n^2 \log n)$ iterations witnesses the full Pareto front. This is the same bound on the runtime that is known for the setting without noise~\cite{GielL10,OsunaGNS20}. We note that in comparable single-objective settings, not much more noise can be tolerated. For example, only for $p = O(\log(\log(n))/n)$ it could be shown in~\cite{Dang-NhuDDIN18} that the \oea retains its noise-free $O(n \log n)$ runtime on the \onemax benchmark. Already for $p = \omega(\log(n)/n)$, the runtime is super-polynomial. Considering this and taking into account that the multi-objective \oneminmax problem is naturally harder (we aim at a population containing exactly one solution for each Pareto optimum), our result indicates that MOEAs cope with noise surprisingly well.

However, our work also shows one important difference to the noisy optimization of single-objective problems. Our result above assumes that each solution is evaluated only once, namely when it is generated. This possibly noisy objective value is stored with the solution unless the solution is discarded at some time. This approach is natural in that it avoids costly reevaluations, but it differs from the standards in single-objective evolutionary computation. Being afraid that a faulty fitness value can negatively influence the future optimization process, almost all works there assume that each time a solution competes with others, its fitness is reevaluated. That noisy fitness values without reevaluating solutions can significantly disturb the optimization process was rigorously shown for an ant-colony optimizer~\cite{DoerrHK12ants}.

Given that in single-objective evolutionary computation it is more common to assume that solutions are evaluated anew when they compete with others, we also analyzed the runtime of the SEMO on \oneminmax under this assumption. While this avoids sticking to faulty objective values for a long time, our mathematical runtime analysis shows that this approach can only tolerate much lower levels of noise. We can prove an $O(n^2 \log n)$ runtime only for $p \le \beta/n^2$, $\beta$ a suitable constant, and we prove that for $p = \omega(\log(n)/n^2)$ the algorithms needs super-polynomial time to compute the full Pareto front of the \oneminmax benchmark. So clearly, the reevaluation strategy recommended in single-objective optimization is less suitable in multi-objective optimization (in addition to the significantly higher computational cost of a single iteration). 

Overall, this first runtime analysis work of a standard MOEA in a noisy environment shows that MOEAs without specific adjustments are reasonably robust to noise and that such stochastic processes can be analyzed with mathematical means, but also that insights from the single-objective setting can be fundamentally wrong in multi-objective noisy optimization.

\section{Previous Work}
Given the success of evolutionary algorithms both in multi-objective and in stochastic optimization, there is a large body of literature on both topics. For the general picture, we refer to the surveys~\cite{CoelloLV07,ZhouQLZSZ11} and~\cite{JinB05,BianchiDGG09}.

Much less understood is the intersection of both areas, that is, how MOEAs solve stochastic optimization problems, see~\cite[Section~4]{Gutjahr11trends}. Interestingly, all works described there adapt the MOEA to deal with the stochasticity of the problem. This is justified for problems that have a strong stochasticity known in advance. However, it is known that evolutionary algorithms often can tolerate a certain amount of stochastic disturbance, in particular, noisy function evaluations without any adjustments (and hence, without that the user has to be aware of this noise). In contrast to single-objective optimization, we are not aware of any such result for MOEAs.

This being a theoretical work, we now describe the mathematical runtime analyses closest to our work. The mathematical runtime analysis of MOEAs was started in~\cite{LaumannsTDZ02,Giel03,Thierens03}. These and most subsequent works regarded artificial toy algorithms like the simple evolutionary multi-objective optimizer (SEMO) or the global SEMO (GSEMO). The hope is that results proven on such basic algorithms, or at least the general insights drawn from them, extend to more realistic algorithms. That this hope is not unrealistic can be seen, e.g., from the fact that the runtimes shown for the SEMO  in~\cite{LaumannsTZWD02,GielL10} could much later be proven also for the \NSGA~\cite{ZhengLD22,BianQ22,DoerrQ23LB}, the most prominent MOEA. 

For our work, naturally, the results on the \oneminmax benchmark are most relevant. In~\cite{GielL10}, it was proven that the SEMO finds the Pareto front of this benchmark in an expected number of $O(n^2 \log n)$ iterations. A matching lower bound was proven in~\cite{OsunaGNS20}. The $O(n^2 \log n)$ upper bound also holds for the GSEMO (never formally proven, but easy to see from the proof in~\cite{GielL10}), the hypervolume-based $(\mu+1)$ SIBEA with suitable population size~\cite{NguyenSN15}, and the \NSGA with suitable population size~\cite{ZhengLD22}. For the NSGA-III, a runtime analysis exists only for the $3$-objective version of \oneminmax \cite{WiethegerD23}, but it is obvious that this result immediately extends to an $O(n^2 \log n)$ bound in the case of two objectives.  The main argument in all these analyses (which is not anymore true in the noisy setting) is that Pareto points cannot be lost, that is, once the population of the algorithm contains a solution with a certain Pareto-optimal objective value, it does so forever. 

Due to their randomized nature, it is not surprising that EAs can tolerate a certain amount of stochastic disturbance. In the first runtime analysis of an EA for discrete search spaces in the presence of noise~\cite{Droste04}, Droste investigated how the \oea optimizes the \onemax benchmark in the presence of bit-wise prior noise with noise rate~$p$ (see Section~\ref{subsec:noise} for a precise definition of this noise model). He showed that the runtime remains polynomial (but most likely higher than the well-known $O(n \log n)$ runtime of the noiseless setting) when $p = O(\log(n)/n)$. When $p = \omega(\log(n)/n)$, a super-polynomial runtime results. These bounds have been tightened and extended in the future, in particular, an $O(n \log n)$ runtime bound for noise rate $p = O(\log(\log(n))/n)$ with implicit constant small enough was shown in~\cite{Dang-NhuDDIN18}. The level of noise an algorithm can stand depends strongly on the problem, for example, for the \leadingones benchmark the \oea has a polynomial runtime if and only if the noise rate is at most $p = O(\log(n)/n^2)$~\cite{Sudholt21}. We refer to~\cite{NeumannPR20bookchapter} for more results.

We note that all these works regard the standard version of the EA without any particular adjustments to better cope with noise. It is well-studied that resampling techniques can increase the robustness to noise~\cite{AkimotoMT15,QianYTJYZ18,DoerrS19}, however, this requires the algorithm user to be aware of the presence of noise and have an estimate of its strength. 

\emph{Reevaluating solutions:} When running a standard EA in a noisy environment, there is one important implementation choice, namely whether to store the possibly noisy fitness of a solution or to reevaluate the solution whenever it becomes important in the algorithm. The latter, naturally, is more costly due to the higher number of fitness evaluations, but since most EAs generate more solutions per iteration than what they take into the next generation, this often is only a constant-factor performance loss. On the other hand, sticking to the objective value seen in the first evaluation carries the risk that a single unlucky evaluation has a huge negative impact on the remaining run of the algorithm.

The question which variant to prefer has not been discussed extensively in the literature. However, starting 
with the first runtime analysis of an EA in a noisy environment~\cite{Droste04}, 
almost all runtime analyses of EAs in a noisy environment assume that each solution is reevaluated whenever it plays a role in the algorithm, in particular, whenever solution qualities are compared. 

The only example we are aware of where this is done differently is the analysis of how an ant-colony optimizer (ACO) solves a stochastic shortest path problem in~\cite{SudholtT12}. This work regards a variant of the Max-Min Ant System~\cite{StutzleH00} with best-so-far pheromone update where the best-so-far solution is kept without reevaluation. While this algorithm was provably able to solve some stochastic path problems, it was not able to solve others (except possibly in exponential time). The same algorithm, but with the best-so-far solution being reevaluated whenever it competes with a new solution, was analyzed in~\cite{DoerrHK12ants} and it was shown to be able to solve many problems efficiently which could not be solved by the other variant. 

\emph{Noisy evolutionary multi-objective optimization:} 
While we apparently have a good understanding on how EAs cope with noise and how they can be used to solve multi-objective problems, there has been very little research on how EAs solve multi-objective problems in the presence of noise. The only mathematical runtime analysis in this direction is~\cite{Gutjahr12}. It analyzes how the \emph{adaptive Pareto sampling} framework together with a variant of the SEMO (called with suitable noiseless instances) allows to solve noisy instances. So like all other non-mathematical works on noisy heuristic multi-objective optimization, e.g.,~\cite{Teich01,DingBX06,LiefoogheBJT07,BoonmaS09,FieldsendE15,RakshitK15}, this work does not analyze how a standard MOEA copes with noise, but proposes a specific way how to solve noisy multi-objective optimization problems.

        
\section{Preliminaries: Multi-objective Optimization of the \oneminmax Benchmark in the Presence of Noise}
\subsection{The \oneminmax Benchmark}

In multi-objective optimization over a set $X$ using an evaluation function $g:X\rightarrow \R^d$, we say that a vector $x\in X$ is a Pareto optimal solution if there is no $y\in X$ with $g(x)\prec g(y)$. We denote as $X^*$ the set of Pareto optimal vectors, and define the Pareto front as $g(X^*)$.
Note that we will use the partial order on $\R^d$ where $x\preceq y$ if and only if for all $i\in[1..d]$, we have $x_i\le y_i$. The induced strict ordering will then be defined as $x\prec y \iff (x\preceq y \land x\neq y)$.

The \oneminmax benchmark is a bi-objective optimization problem defined over the decision space $X = \{0, 1\}^n$.
It was introduced in \cite{GielL10} as a bi-objective version of the classic \onemax benchmark. The objective function is defined as
\[
    g(x) = \left(f(x), n - f(x)\right),
\]
where the first objective $f(x) = \onemax(x)$ is the number of ones in $x$. The goal of the algorithms analyzed in this paper is to find a minimal subset of $X$ which has a direct image by $g$ equal to the Pareto front of $X$ for this benchmark. Since any solution to the \oneminmax benchmark is Pareto optimal, the Pareto front of this problem is $X^* =  \{(k, n - k)\mid k\in [0..n]\}$.
    
\subsection{One-bit Noise}
\label{subsec:noise}
In this article, we consider that every evaluation of an objective vector is subject to noise. Let $p \in [0, 1]$ be the noise rate in the whole article. 

We define the noisy vector evaluation function $\tilde{x}$ as the vector $x$ but where a random bit uniformly chosen has been flipped with probability $p$. Each noisy evaluation is independent from one another. This noise model is known as \textit{one-bit noise}.

For any function $\phi : X \rightarrow Y$, we define
\[
    \tilde{\phi}(x) \coloneqq \left\{
    \begin{array}{ll}
        X \rightarrow Y \\
        x \mapsto \phi(\tilde{x})
    \end{array}.
    \right.
\]
It is therefore possible to replace the objective function $g$ by its noisy version $\tilde{g}$ in any evaluation.

\subsection{Modified SEMO Algorithms for Optimizing Noisy Objective Functions}

When running an EA on a noisy problem, one can work with the first noisy fitness value received for a solution or to reevaluate it each time it is relevant. We present two versions of the same algorithm for these two possibilities.
In all our algorithms, populations $\calP_t$ or $\calP'_t$ are multisets. This is necessary even for a SEMO algorithm (which stores at most one solution per objective value, and hence not multiple copies of an individual) since the noise fitness may let copies of an individual appear different due to their noise fitness.

\subsubsection{SEMO Without Reevaluation}
The first algorithm is the original SEMO algorithm, but it evaluates values of an element entering the population only once and save its value as long as this element stays in the population. 

In Algorithm~\ref{alg:SEMO_without_reevaluation}, we store the noisy evaluations of elements of $\calP_t$ in $\calW_t$. The function $\calW_t$ can be seen as a map that associates each element of $\calP_t$ with its noisy value as computed at the time of its addition.

\begin{algorithm}[t]
    \caption{SEMO without reevaluation}
    \label{alg:SEMO_without_reevaluation}
    \DontPrintSemicolon
    \SetKwBlock{Repeat}{repeat}{}

    $x_{\init}$ is uniformly chosen from $\{0, 1\}^n$ \;
    $\calP_0 \gets{\{x_{\init}\}}$ \;
    $\calW_0(x_{\init}) \gets \tilde{g}(x_{\init})$ \;
    \;

    \For{$t = 0$ to $\infty$}{
        Choose $x_t$ uniformly from $\calP_t$ \;
        Sample $x_t'$ from $x_t$ by flipping one uniformly chosen bit in $x_t$\;
        \;
        $w \gets{\tilde{g}(x_t')} $ \;
        \;
        \If{there is no $x\in\calP_t$ such that $w \prec \calW_t(x)$}{
            $\calP_{t + 1} \gets{\{x_t'\}}$ \;
            $\calW_{t + 1}(x_t') \gets{w}$ \;
            \;
            \ForEach{$x\in\calP_t$}{
                \If{$\calW_t(x) \npreceq w$}{
                    $\calP_{t + 1}\gets{\calP_{t + 1}\cup\{x\}}$\;
                    $\calW_{t + 1}(x) \gets{\calW_t(x)}$
                }
            }
        }
        \Else{
            $\calP_{t + 1} \gets{\calP_{t}}$ \;
            $\calW_{t + 1} \gets{\calW_{t}}$ \;
        }
    }
\end{algorithm}

Note that we state this algorithm in a general way to deal with any multi-objective optimization. For \oneminmax, this algorithm could be simplified, since the condition in line \textbf{11} is
always true and the inequality in line \textbf{16} becomes an equality. 

\subsubsection{SEMO With Reevaluation}

We now describe a SEMO algorithm which reevaluates each solution in each iteration. Reevaluations have been recommended in noisy settings to avoid working with the same, faulty fitness value for a long time. For the SEMO, this requires some adjustments since now it does not suffice anymore to integrate the new offspring into the existing population. Due to the changing fitness values of individuals, it may also happen that two old solutions appear equal, so one of the two has to be removed.

\begin{algorithm}[t]
    \caption{SEMO with reevaluation}
    \label{alg:SEMO_with_reevaluation}
    \DontPrintSemicolon
    \SetKwBlock{Repeat}{repeat}{}

    $\calP_0 \gets{\{x_{\init}\}}$, where $x_{\init}$ is uniformly chosen from $\{0, 1\}^n$ \;
    \;

    \For{$t = 0$ to $\infty$}{
        Choose $x_t$ uniformly from $\calP_t$ \;
        Sample $x_t'$ from $x_t$ by flipping one uniformly chosen bit in $x_t$ \;
        \;
        $\calP_t' \gets{\calP_t\cup\{x_t'\}}$ \;

        $\calP_{t+1}\gets{\elim(\calP_t', \tilde{g})}$
    }
\end{algorithm}

This algorithm uses the $\elim$ function defined below. Given a multiset $E\subset X$ and a possibly random evaluation function $h : X \rightarrow Y$
where Y is a partially ordered set, $G\coloneqq \elim(E,h)$ will be a minimal sub-multiset of $E$ that Pareto-dominates $E$. Pareto-dominance meaning
that for every $x\in E$, there exists $y\in G$ with $h(x) \preceq h(y)$, and minimality meaning that for every $x,y\in G$, $h(x)$ and $h(y)$ are not
comparable.

\begin{algorithm}[t]
    \caption{Elimination function $\elim$. Inputs are the initial set $E$ and the evaluation function $h$. Output is $G$, a minimal Pareto-dominant
        sub-multiset of $E$.}
    \label{alg:elim}
    \DontPrintSemicolon
    \SetKwBlock{Repeat}{repeat}{}

    $G \gets{\emptyset}$ \;
    $\Temp \gets{\emptyset}$ \;
    \;

    \ForEach{$x\in E$ selected in a random order}{
        $v_x \gets{h(x)}$\;
        \;
        \If{there is no $v\in \Temp$ such that $v_x \prec v$}{
            \ForEach{$y$ in $G$ with $v_y \preceq v_x$}{
                $G\gets{G \setminus \{y\}}$\;
            }
            $G\gets{G\cup\{x\}}$\;
            $\Temp\gets{\Temp\cup \{v_x\}}$
        }
    }
    \Return{G}
\end{algorithm}
For \oneminmax, the $\elim$ function can be simplified, with the condition in line \textbf{7} being always verified and the inequality in line \textbf{8} being an equality.


\section{Runtime Analysis of the SEMO Without Reevaluation}

In this section, we analyze the expected number of iterations of the main loop needed to reach the Pareto front in the SEMO without reevaluation. 
Our main technical tool is drift analysis, that is, a collection of mathematical tools that allow to translate estimates for the one-step change $X_{t+1} - X_t$ of a random process $(X_t)$ into estimates for hitting times of this process. See~\cite{Lengler20bookchapter} for a recent survey on this method. Unfortunately, for reasons of space, we cannot present the proofs of our results in full detail. They can be found in the long version \cite{DinotDHW23arxiv}.

We define the total time to reach the Pareto front as a stopping time
\[
    T_{\total} \coloneqq \min\{t \mid f(\calP_t) = [0..n]\}
\]
and show that there is a constant $\alpha$ such that for all noise rates $p\le\tfrac{\alpha}{n}$ we have $E[T_{\total}] = O(n^2\log(n))$.

\subsection{Observations on the SEMO Without Reevaluation}
The function $\calW_t$ is the map from the elements of $\calP_t$ to their noisy evaluated value at the moment they entered the population. We call $\calV_t$ the first component of $W_t$, that is $W_t = (\calV_t, n-\calV_t)$. 
At time $t$, two situations can occur.
\begin{itemize}
    \item The noisy value of the offspring $x_t'$ is not in $\calV_t(\calP_t)$, so it is added to the population and no element is removed. 
    \item There is an element $x\in\calP_t$ such that $\calV_t(x) = \calV_{t + 1}(x_t')$, so the value is already in the set. In that case, $x$ will be removed from the population and $x_t'$ will be added.
\end{itemize}
From these observations, we deduce the two following lemmas. 
\begin{lemma}
    \label{lemma:increasing}
    $\calV_t(\calP_t)\subset\calV_{t + 1}(\calP_{t + 1})$
\end{lemma}
\begin{proof}
    An element is removed from the population only if the offspring has the same value.
\end{proof}
\begin{lemma}
    \label{lemma:one-to-one}
    $\calV_t$ is one-to-one.
\end{lemma}
\begin{proof}
    It is true at time $t=0$ because there is only one vector in $\calP_t$. If it is true at time $t$,
    then the only possible collision is with the offspring $x_t'$. In that case the element of the same value will be removed from the population and the offspring will be added. $V_{t +1}$ is still one-to-one. The result is true by induction.
\end{proof}

Another important lemma is the following.

\begin{lemma}
    \label{lemma:maxstack}
    $\forall t, \forall k\in[0..n], |\{x\in\calP_t \mid f(x) = k\}| \le 3$.
\end{lemma}
\begin{proof}
    If $f(x) = k$, then because of one-bit noise $\calV_t(x)\in\{k - 1, k, k + 1\}$.
    Since $\calV_t$ is one-to-one, there are at most $3$ elements of value $k$ by $f$.
\end{proof}
\subsection{Time Needed to Find the Extreme Values}
We define 
\[
    T_{\ex} = \min\{t \mid \{0,n\}\subseteq\calV_t(\calP_t)\},
\]
the first time when the extreme values are noisily found.
\begin{theorem}
    \label{theorem:reaching_borders}
    There is a constant $\alpha > 0$ such that if $p \le \tfrac{\alpha}{n}$, then
    \[
        E[T_{\ex}] = O(n^2\log(n)).
    \]
\end{theorem}
\begin{proof}
    We denote as $j_t$ the minimum value of $\calV_t(\calP_t)$ and we introduce the potential function $\ell_t = j_t - \bbone{\exists x\in\calP_t, \calV_t(x) = f(x) = j_t} + 1$. This variable equals $j_t+1$, but with a bonus downward if the unique element $x$ with $\calV_t(x)=j_t$ is correctly evaluated, that is $f(x) = j_t$. We analyze the drift, that is the expected change of the process $\ell_t$.
    
    We consider first the case where $f(x) = j_t$. 
    \begin{itemize}
        \item The probability of selecting $x$ as a parent for mutation is $\tfrac{1}{|\calP_t|}$. The probability that its offspring has a lower value is $\tfrac{j_t}{n}$ (namely, selecting a $1$ in $x$ and flipping it to a $0$). The probability of correctly evaluating the offspring is $1 - p$. The overall probability of this event is $\tfrac{j_t(1 - p)}{n|\calP_t|}$. In that case $\ell_t - \ell_{t + 1} = 1$. Therefore, when $f(x) = j_t$, the downward component of the drift is at least $\tfrac{j_t(1 - p)}{n|\calP_t|}$.
        \item  For the upward component of the drift, since $j_{t + 1} \le j_t$ as stated in Lemma~\ref{lemma:increasing}, a negative shift only happens if the offspring $x_t'$ is wrongly evaluated to $j_t$. We would have $j_{t + 1} = j_t$ but $\bbone{\exists x\in\calP_{t + 1}, \calV_{t + 1}(x) = f(x) = j_t} = 0$ and $\ell_t - \ell_{t + 1} = -1$. Since in general $\calV_t(x)\in\{f(x) - 1, f(x), f(x) + 1\}$, the value of the offspring is $j_t - 1$ or $j_t + 1$. The mutation is also one-bit, so the parent $x_t$ at time $t$ must be of value $j_t - 2$,  $j_t$ or $j_t + 2$. The value $j_t - 2$ is not an option, as the noisy value of $x_t$ would be lower than $j_t - 1$ so strictly lower than $j_t$, so there is at most $6$ elements verifying this property according to Lemma~\ref{lemma:maxstack}. The probability of selecting one of those to create the offspring is lower than $\tfrac{6}{|\calP_t|}$. The probability of this offspring to be wrongly evaluated to $j_t$ is lower than $p$.
    \end{itemize} 
    The total drift when $f(x) = j_t$ is $\tfrac{j_t(1 - p)}{n|\calP_t|} - \tfrac{6p}{|\calP_t|}$.

    Now consider the case where $f(x) \neq j_t$. In that case, $\ell_t - \ell_{t + 1} \ge 0$ so the drift can only be downward. With similar arguments, we can show that in that case the drift is at least $\tfrac{j_t(1 - p)}{n|\calP_t|}$ which is greater than $\tfrac{j_t(1 - p)}{n|\calP_t|} - \tfrac{6p}{|\calP_t|}$.
    
    In both cases, the drift is greater than 
    \[\frac{j_t(1 - p)}{n|\calP_t|} - \frac{6p}{|\calP_t|}.\] This is enough to conclude by the multiplicative drift theorem~\cite{DoerrJW12algo} that the lowest value is reached in $O(n^2\log(n))$ when $p\le\frac{1}{14n}$. As a consequence of Lemma~\ref{lemma:increasing}, it will remain in $\calV_t(\calP_t)$. Symmetrically, the time to reach the highest value is identical.
\end{proof}

\subsection{Filling the Pareto Front Once the Extreme Values are Found}

To derive the total time, we additionally analyze the average time to fill the full Pareto front after the extreme values are found.

\begin{theorem}
    We consider the same algorithm, but where the extreme elements are initially in the set, that is ${0, n}\in\calV_0(\calP_0)$. We note this event $Y_0$.
    Then there exists a constant $\alpha>0$ such that
    \[
        E_{Y_0}[T_{\total}] =O(n^2).
    \]
\end{theorem}

\begin{proof}
   The proof relies on the fact that since the extreme values have been found, we know there are elements $x, y\in\calP_t$ verifying $\calV_t(x) = 0, \calV_t(y) = n$ so $f(x) \le 1$ and $f(y)\ge n - 1$. 
   
   If $\left[1..\left\lfloor \tfrac{n}{2} \right\rfloor\right] \setminus f(\calP_t)$ is non-empty, then let us define \[m = \min\left(\left[1..\left\lfloor \tfrac{n}{2} \right\rfloor\right] \setminus f(\calP_t)\right).\] Since $x\in\calP_t$, $\left[0..\left\lfloor \tfrac{n}{2} \right\rfloor\right] \cap f(\calP_t)$ is non-empty. So there is an element $z\in\calP_t$ such that $f(z) = m - 1$. The probability of selecting it to generate the offspring $x_t'$ is $\tfrac{1}{|\calP_t|}\ge\tfrac{1}{n + 1}$. The probability that this offspring ends up with a real value $m$ is $\frac{n - f(z)}{n}\ge\frac{1}{2}$. The probability of evaluating correctly this offspring is $1 - p$. A symmetrical argument can be made for $\left[\left\lceil  \tfrac{n}{2}  \right\rceil ..n - 1\right] \setminus f(\calP_t)$ which yields the same result. 
   
   Hence, we see that the variable counting elements in the population that are not extreme elements and are correctly evaluated features an upward drift of at least $\frac{1 - p}{2n}$. For the downward drift, observe that no more than one element is deleted at each step, and that elements that are correctly evaluated can only be lost if the offspring is miss-evaluated, which happens with probability $p$. Thus, the total drift on this count variable is at least $\frac{1 - p}{2n} - p$.
   
   As for the variable counting correctly evaluated extreme elements, an argument similar to the proof of Theorem~\ref{theorem:reaching_borders} shows it features a positive drift component of at least $\frac{1}{n+1}\frac{1}{n}(1-p)$ and a negative drift of no more than $\frac{12p}{n}$, and so a total drift of at least $\frac{1}{n+1}\frac{1}{n}(1-p) - \frac{12p}{n}$. 
   
   For a well-chosen constant $\alpha > 0$, if $p\le \tfrac{\alpha}{n}$ then both drifts are positive. By considering a weighted sum of the two count variables, with a weight of $1$ for the first and a weight of $n$ for the second, the additive drift theorem~\cite{HeY01} yields the desired result.
\end{proof}

So, the total runtime of the SEMO without reevaluation is $O(n^2\log(n)) + O(n^2) = O(n^2\log(n))$ as announced.

We observe that this second phase of the run is finished in a relatively short time. The reason for this is that once we have the extremal elements in the population, we can generate missing solutions (in particular, those which have been lost due to faulty fitness evaluations) from a neighboring solution closer to the boundaries. This is relatively easy since mutation has a constant probability to generate the desired offspring from such a parent. Here we profit from the diversity mechanism of the SEMO, which tries to keep one solution per objective value in the population. This seems to be the main reason for the relatively high robustness of the SEMO to noise, and we would speculate that other MOEAs profit from their diversity mechanisms (needed to construct the whole Pareto front) in a similar fashion.

\section{Runtime Analysis of the SEMO With Reevaluation}
In this section, we analyze the expected number of iterations of the main loop needed to reach the Pareto front in the SEMO with the SEMO with reevaluation~\ref{alg:SEMO_with_reevaluation}. We define the stopping time 
\begin{align*}
    T_{\total} &\coloneqq \min\{t \mid g(\calP_t) \mbox{ is equal to the Pareto front}\} \\ &= \min\{t \mid f(\calP_t) = [0..n]\}
\end{align*}   
and show the following theorem.

\begin{theorem*}
There is a constant $\beta>0$ such that the following holds. For all noise rates $p \le \tfrac{\beta}{n^2}$, the SEMO with reevaluations finds the complete Pareto front of the \oneminmax benchmark in expected time
\[
E[T_{\total}] = O(n^2\log(n)).
\]
\end{theorem*}
We will assume in the whole section that $p \le \tfrac{\beta}{n^2}$, with $\beta>0$ a small enough constant.
We use $\calP_t$ for the population at time $t$, $\calP_t' = \calP_t \cup \{x_t'\}$ for the population plus the mutated vector. $L_t$ is the cardinality of $f(\calP_t)$ and $L_t'$ the cardinality of $\calP_t'$.
Some of these notations are already
used in the definition of  Algorithm~\ref{alg:SEMO_with_reevaluation}.

\subsection{Time to Find the Extreme Values}
Here, we define $T_{\ex} = \min\{t\mid 0, n\in f(\calP_t)\}$ in terms of real values. Contrary to the other algorithm, these extreme values can be lost because of the noisy reevaluations. We show that still $E[T_{\ex}]=O(n^2\log(n))$. 

To prove this result, we consider the variable $d_t = n + \min f(\calP_t) - \max f(\calP_t)$ and apply the multiplicative drift theorem~\cite{DoerrJW12algo} to it. 
We note that the mutation step gives a drift (towards zero) of order at least $\tfrac{d_t}{n^2}$, same as in a non-noisy setting. Different from the noiseless case, the selection can then lead to drift away from zero of order $\calO(p)$. Note that in any case, $d_t$ can change by at most one. By our assumption on $p$ the drift into the wrong direction is small compared to the drift towards zero, so asymptotically we have a ``multiplicative'' drift of order at least $\tfrac{d_t}{n^2}$, which allows a straight-forward application of the multiplicative drift theorem to prove the announced result.

\subsection{Computing the Pareto Front Once the Extreme Values are Found}
We deal with the last phase by regarding an artificially modified process of the original SEMO with reevaluation. We denote this new process as the $K$-extreme values keeping SEMO, $K\in \N \cup \{+\infty\}$. This process is identical to the SEMO with reevaluation, except that it will tzappend the extreme elements to the population at the end of each iteration if they have been lost and if less than $K$ iterations have been executed. Variables related to this new process feature a hat and a super-index $K$ (e.g. $\hat{\calP}_t^K$). In the whole subsection, we will use the event $Z_0 = (\{0,n\}\subset f(\calP_0))$. We also define $L_t$ as the cardinality of $f(\calP_t)$.

By using the additive drift theorem \cite{HeY01} on the variable $\hat{L}_t^K$, we show the following lemma.
\begin{lemma}
    \label{lemma:markov}
    \[
        {\Pr}_{Z_0}[\hat{T}_{total}^K \ge K] = \calO\left(\frac{n^2}{K}\right).
    \]
\end{lemma}

The following lemma shows that the two algorithms behave in the same way with high probability.
\begin{lemma}
    \label{lemma:diff process}
    For any $K\in \N$, we have
    \[
        {\Pr}_{Z_0}[T_{\total} = \hat{T}_{total}^K] \ge (1-p)^{10K}.
    \]
\end{lemma}

Indeed, of the classical process does not lose its extreme values in the $K$ first iterations, then it is identical to the $K$-extreme values keeping process.

We can finally deduce the last result.
\begin{lemma}
    \label{lemma:full phase}
    There exist constants $M>0$ and $\mu\in (0,1)$ such that for any $n\in\N$, with $t_0=Mn^2\log(n)$ and any distribution $\pi$ of the initial population $\calP_0$,
    \[
        {\Pr}_\pi[T_{\total} > t_0] \le \mu.
    \]
\end{lemma}
This lemma states that the probability for the process, starting in any position, to end in time $t_0=\calO(n^2\log(n))$ time is bounded by a constant. Therefore, the average number of phases of length $t_0$ is constant, hence the final result.

\section{Lower Bound on the SEMO With Reevaluation}
In this section, we will show that the previous $\tfrac{\beta}{n^2}$ bound for $p$ for the SEMO with reevaluation to run in $O(n^2\log(n))$ is nearly tight. Namely, if $p=\omega(\tfrac{\log(n)}{n^2})$ and for some constant $\lambda\in(0,1)$ we have $p\le\tfrac{\lambda}{n}$, then the runtime is super-polynomial.

Here, we will use the notation from the previous section. Our main ingredient for the proof will be the simplified drift theorem~\cite{OlivetoW11,OlivetoW12} used on the random sequence $(L_t)_{t\in\N}$, where $L_t=|f(\calP_t)|$.

In the whole proof, we will assume that $p=\omega(\tfrac{\log(n)}{n^2})$ and that there exists a constant $\lambda\in]0,1[$ such that $p\le\tfrac{\lambda}{n}$.

\begin{lemma}
    \label{lemma:backward drift}
    We denote $a_n = \max(n+1-\beta pn^2, \tfrac{3n}{4})$, with $\beta > 0$ to be determined. There exists a constant $\delta>0$ such that 
    $ 
        E[L_t-L_{t+1}\mid L_t, L_t \ge a_n] \ge \delta np.
    $
\end{lemma}
This lemma indicates the presence of a negative drift of $L_t$. Indeed, when that variable is close to $n$, each wrong evaluation will cause on average a deletion. This effect will dominate the gains from the mutation if the noise rate is big enough.
We then show that
    \[ 
        \Pr[|L_t-L_{t+1}|\ge j\mid L_t]=O_{j\rightarrow +\infty}(\lambda^j).
    \]
These are the two hypotheses for the Simplified Drift Theorem, as stated in \cite{OlivetoW12}. The theorem implies the existence of a constant $B>0$ such that $\Pr[T_{\total}\le 2^{Bpn^2}]\le 2^{-\Omega(pn^2)}$. As $pn^2=\omega(\log(n))$, this yields the final result.
        
\section{Conclusion}

We conducted the first mathematical runtime analysis, and to the best of our knowledge also the first analysis at all, aimed at understanding how MOEAs without particular adjustments behave in the presence of noise. Taking the SEMO algorithm and the \oneminmax benchmark with one-bit noise as an example simple enough to admit a mathematical analysis, we show that noise rates up to $p = O(1/n)$ can be tolerated without suffering from an asymptotic increase of the expected runtime (whereas runtime we regard the time it takes until for the first time the population of the algorithm witnesses the full Pareto front). This robustness is very close to the $p=O(\log(\log(n))/n)$ noise tolerance known for the much simpler single-objective \onemax problem and suggests that MOEAs, despite the more complex problem of finding the whole Pareto front, can stand similar noise levels as single-objective EAs.

Interestingly, our result only holds when solutions are evaluated only once and the algorithm then continues with this, possibly faulty, objective value. If solutions are reevaluated whenever they are compared to other solutions (here in each iteration), then any noise rate $p = \omega(\log(n)/n^2)$ leads to a super-polynomial runtime. This is in drastic contrast to single-objective evolutionary computation, where the general recommendation (backed up by an analysis of an ant colony optimizer) is to reevaluate each solution when its fitness is relevant so that a single faulty fitness evaluation cannot harm the future optimization process for a long time. 

From the proofs of our results we are optimistic that our general finding that MOEAs without particular adjustments are robust to a moderate amount of noise is not restricted to the particular setting we analyzed, but holds for a broad range of algorithms and benchmarks. A reasonable next step to support this belief would be to regard the global SEMO algorithm (using bit-wise mutation instead of one-bit flips) and a bit-wise noise model (where the fitness observed for $x$ is the fitness of a search point obtained from $x$ by flipping each bit independently with probability $q=p/n$). With this scaling, we would expect very similar results to hold as shown in this work. We note that this appears to be a problem very close to ours, but past research has shown that both global mutation in MOEAs and bit-wise noise can render mathematical analyses much harder (different from the SEMO, there is no good lower bound for the global SEMO on the \lotz benchmark~\cite{DoerrKV13}; the first mathematical runtime analysis in the presence of noise~\cite{Droste04}, namely one-bit noise, was extended to bit-wise noise only more than ten years later~\cite{GiessenK16} and with much deeper methods).

A second interesting direction to extend this work would be by regarding other problems. The \oneminmax problem has the particular properties that all solutions lie on the Pareto front and the one-bit flips are sufficient to explore the whole front. Hence analyses for the \cocz or \lotz benchmarks having non-Pareto optimal solutions \cite{LaumannsTZ04} or the \ojzj benchmark having larger gaps in the Pareto front \cite{DoerrZ21aaai} would be very interesting. Understanding how existing analyses for combinatorial optimization problems such as \cite{Neumann07,CerfDHKW23} extend to the noisy setting could be a subsequent step.

A third interesting direction for future research would be to regard posterior noise models (where the noisy objective value of a search point is its true objective value modified in some stochastic way, e.g., by adding Gaussian noise). For single-objective evolutionary algorithms, often comparable results could be obtained with similar arguments~\cite{GiessenK16,Dang-NhuDDIN18}. Since posterior noise can lead to objective values that cannot be obtained as noiseless fitness values, we have some doubts that our results extend to posterior noise as well. In particular, we would speculate that for posterior noise, different from what our result suggests for prior noise, reevaluating search points is the better strategy as otherwise it is not clear how to remove a noisy objective value that cannot occur as noiseless one.

\section*{Acknowledgments}
This work was supported by a public grant as part of the Investissements d'avenir project, reference ANR-11-LABX-0056-LMH, LabEx LMH and a fellowship via the International Exchange Program of \'Ecole Polytechnique.






{\small
\bibliographystyle{named}
\bibliography{bibliography/ich_master.bib,bibliography/alles_ea_master.bib,bibliography/rest.bib}
}

\end{document}